\documentclass[10pt, conference, compsocconf]{IEEEtran}
\ifCLASSINFOpdf
\else
\fi

\usepackage{xcolor}

\usepackage{algorithm}
\usepackage{algorithmic}
\usepackage{amsfonts}
\usepackage[bookmarks=false]{hyperref}
\usepackage{url}
\usepackage{amsmath}
\usepackage{amsthm}
\usepackage{graphicx}
\usepackage{subcaption}
\newtheorem{theorem}{Theorem}

\newcommand{\zap}[1]{}
\newcommand{\acronym}{DTD-DP} 
\newcommand{\MUCB}{DTD-UCB} 
\newcommand{\name}{DTD-Dynamic-Programming} 
\newcommand{\NMUCB}{DTD-Upper-Confidence-Bound} 

\newlength{\bibitemsep}
\newlength{\bibparskip}
\let\oldthebibliography\thebibliography
\renewcommand\thebibliography[1]{%
  \oldthebibliography{#1}%
  \setlength{\parskip}{\bibitemsep}%
  \setlength{\itemsep}{\bibparskip}%
}

\begin{document}
%
\title{The Bayesian Linear Information Filtering Problem}


\author{\IEEEauthorblockN{Bangrui Chen}
\IEEEauthorblockA{School of Operations Research \& Information Engineering\\
Cornell University\\
Ithaca, NY, 14853, USA\\
bc496@cornell.edu}
\and
\IEEEauthorblockN{Peter I. Frazier}
\IEEEauthorblockA{School of Operations Research \& Information Engineering\\
Cornell University\\
Ithaca, NY, 14853, USA\\
pf98@cornell.edu}
}


%


\maketitle

\begin{abstract}
We present a Bayesian sequential decision-making formulation of the information filtering problem, in which an algorithm presents items (news articles, scientific papers, tweets) arriving in a stream, and learns relevance from user feedback on presented items. We model user preferences using a Bayesian linear model, similar in spirit to a Bayesian linear bandit. We compute a computational upper bound on the value of the optimal policy, which allows computing an optimality gap for implementable policies.
We then use this analysis as motivation in introducing a pair of new Decompose-Then-Decide (DTD) heuristic policies, \name\ (\acronym) and \NMUCB\ (\MUCB). 
We compare \acronym\ and \MUCB\ against several benchmarks on real and simulated data, demonstrating significant improvement, and show that the achieved performance is close to the upper bound.

\end{abstract}


%
\IEEEpeerreviewmaketitle

\section{INTRODUCTION}

Information filtering systems automatically distinguish relevant from irrelevant items (emails, news articles, intelligence information) in large information streams \cite{pid}. They typically use a classifier trained on relevance feedback from past items. However, when filtering for new users, or when item contents or user interests change, sufficient training data may not be available. In such ``cold-start'' situations, it may be beneficial to actively explore user interests by forwarding those items whose relevance we wish to learn, but too much exploration degrades short-term performance. This is an example of the so-called exploration vs. exploitation tradeoff \cite{suttonbarto}.

In this paper, we present a Bayesian sequential decision-making formulation of this problem, where user interests are described by a Bayesian linear model, similar in spirit to a Bayesian linear bandit \cite{tscblp}. The first contribution of our paper is to construct an instance-specific computational upper bound on the value of a Bayes-optimal strategy, which may be used to bound the optimality gap for implementable heuristic policies. 
Our upper bound is most naturally applied to items whose features are weights from a topic model \cite{tmitm} or other mixture model, but can also be applied to other linear models.
Our second contribution is to use the idea of decomposing the problem into a collection of forwarding problems with one-dimensional feature ``vectors", developed in the construction of the upper bound, to create a pair of heuristic policies, jointly given the name {\it Decompose-Then-Decide (DTD)}.  The first heuristic, called \name\   (\acronym), solves each one-dimensional forwarding problem using stochastic dynamic programming, while the second, called \NMUCB\  (\MUCB), uses the upper confidence bound policy with a learning parameter that is adjusted based on the distribution of feature vectors in the given direction.
Finally, we evaluate our upper bound and proposed policies on real and simulated data, and find that our upper bound is typically tight, and that \MUCB\ outperforms a number of benchmarks, including UCB and Linear Thompson Sampling, in all problem instances. 

The traditional approach to adaptive information filtering trains on historical feedback and does not actively explore to get the most useful feedback. However, there has been some work on active exploration in information filtering.
\cite{eeaf} studies a Bayesian decision-theoretic version of this problem in which a univariate score is observed for each item, and relevance is related to this score via logistic regression. The system does active exploration by valuing the information that results from forwarding, via a one-step lookahead calculation.  The multi-step Bayes-optimal policy is not calculated or characterized.
\cite{eeifp} studies another Bayesian decision-theoretic version of this problem in which items are described by a hard clustering scheme, and users have independent heterogeneous preferences for item clusters.  A computational procedure for calculating the (multi-step) Bayes-optimal policy is provided.  However, the learning scheme used does not allow learning user interest in one category from interactions with other related categories, making it difficult to scale to fine-grained item representations.

A much larger literature on active exploration may be found in work on the multi-armed bandit problem \cite{sasde}. Indeed, the information filtering problem we study can be seen as a special case of the (Bayesian) contextual linear multi-armed bandit problem [3, 8, 9, 10]. The context is the feature vector for the arriving paper, and two arms are available: pulling the first arm corresponds to forwarding the paper, and provides a reward corresponding to the paper's relevance, minus some cost for the user's time; pulling the second arm corresponds to discarding the paper, and has known value $0$. 



While much of the work on multi-armed bandits, including work specifically on linear and contextual bandits, has focused on asymptotic regret guarantees when latent parameters (in our case, the vector of user preferences for features) are chosen by an adversary,
we focus on the Bayesian setting, where we assume that latent parameters are drawn from a prior probability distribution.

Our assumption of a Bayesian framework has advantages and disadvantages.
The main advantage is that it supports good performance when the amount of feedback received is small (of great importance in the cold-start setting).  In contrast, algorithms designed to have regret with an optimal rate in the linear bandit setting, such as the PEGE algorithm in \cite{lpb},  typically need a number of interactions at least as large as the dimension of the feature vector, which may be hundreds of dimensions or more.
A Bayesian algorithm can do well much sooner than this, by using information embedded in the prior that, for example, most users have little preference for a particular feature, or that users who prefer one feature tend to not prefer another feature.

The main disadvantage of the Bayesian framework is that choosing a reasonable prior typically requires work and assumptions. 
However, in the specific application context that we study, personalized information filtering, there is a natural way to build a prior from historical interaction data with other users.  We explain and illustrate this method in Section~\ref{sec:realdata1} using the Yelp academic dataset \cite{yelp} and Section~\ref{sec:realdata2} using the arXiv \cite{arxiv} condensed matter dataset.


Our upper bound is an instance-specific computational upper bound on the performance of the optimal policy. It can be used to compute how far \acronym, \MUCB, or any other policy is from optimal for any given problem instance by computing the value of the heuristic with simulation, computing the upper bound, and subtracting the value from the bound.  In industry, where one must allocate engineering and data science effort across projects, and one typically has a collection of concrete problems with business impact, this supports deciding whether the improvements that will be seen from continued algorithmic development are worthwhile, or whether the best existing heuristic is good enough.  While our upper bound does not determine whether a proposed algorithm attains the optimal asymptotic rate, nor does it allow computing worst-case bounds over all problem instances, we argue that knowing distance from the optimal finite-time performance for specific problem instances with business impact is often more useful.

This paper is structured as follows. In Section~\ref{sec:prob}, we formulate the Bayesian information filtering problem. In Section~\ref{sec:theory}, we develop a computationally tractable upper bound on the value of an optimal policy (Section~\ref{sec:bound}), use this analysis to motivate development of \acronym\ (Section~\ref{sec:policy1}) and \MUCB\ (Section~\ref{sec:policy2}). In Section~\ref{sec:numerical} we compare \acronym\ and \MUCB's performance against benchmarks on both real and simulated data, show a significant improvement over the best of these benchmarks, tuned UCB, and show that its performance is close to the computational upper bound across a range of problems.

\section{Problem formulation}
\label{sec:prob}

We consider information filtering for a single user. Items arrive to the information filtering system following a Poisson distribution with rate $\Gamma$. The \textrm{$n^{\mathrm{th}}$} arriving item is described by a k-dimensional feature vector $X_{n}=(x_{1,n},\cdots,x_{k,n})$. 
We assume that $x_{i,n}\geq 0$ for all $i$ and $n$ (If $x_{i,n}$ are bounded below, then this is without loss of generality). The vector $X_{n}$ is observable to the system when the item becomes available for forwarding, and we assume the system also knows the distribution of $X_{n}$. This distribution can typically be estimated from historical data. In this paper, we denote the density function of the feature vectors' distribution as $f(X_n)$.

Let $\theta=(\theta_{1},\cdots,\theta_{k})$ denote the single user's latent preference vector for the $k$ different features. Here we model $\theta$ as having been drawn from a multivariate normal distribution with mean $\mu_{0}=(\mu_{1,0},\cdots,\mu_{k,0})$ and covariance matrix $\Sigma_{0}$, which represents our Bayesian prior distribution about the latent preference vector. Usually this initial belief can be obtained using the historical data from other users and we give examples of how this may be accomplished in Section~\ref{sec:realdata1} and Section~\ref{sec:realdata2}.  Further, we use $\mu_{n}$ and $\Sigma_{n}$ to denote our Bayesian posterior distribution about the user's reward vector after the arrival of the first n items.

Upon each item's arrival, the system decides whether to forward this item to the user or not. We let $U_{n}\in \{0,1\}$ represent this decision for the \textrm{$n^{\mathrm{th}}$} item, where 1 means to forward and 0 means not to forward. If the system decides not to forward, then the item is discarded. Each time the system forwards, it pays a constant cost $c$ and receives the item's relevance $Y_n$ as a reward.  This relevance is modeled as the inner product between the user's unobservable vector of preferences for features $\theta$ and the item's feature vector $X_n$, perturbed by independent normal noise $\epsilon_n$ with variance $I(X_{n})\lambda^2$, where $I(X_{n})$ denotes the number of non-zero elements in $X_{n}$.  The system only observes $Y_{n}$ if it forwards the item. Except for the fact that some $Y_n$ are unobserved, this statistical model is Bayesian linear regression (see \cite{bda}, Chapter 14).

In many applications, $I(X_n)=k$ with probability 1, making our assumed observational variance of $I(X_n)\lambda^2$ equivalent to assuming homogeneous variance $k \lambda^2$.  Even when $I(X_n)$ varies, we may modify our problem by perturbing each component of $X_n$ by some arbitrarily small $\epsilon>0$ to make $I(X_n)=k$ without substantially affecting the value of any particular policy.

The decision of whether or not to forward the \textrm{$n^{\mathrm{th}}$} item can only depend on the previous information $H_{n-1}=(U_{m},X_{m},U_{m}Y_{m}:m\leq n-1)$ as well as our current $X_{n}$. A policy $\pi$ is a sequence of functions $\pi=(\pi_{1},\pi_{2},\cdots)$ such that $\pi_{n}=(\mathbb{R}_{+}^{k}\times \{0,1\})^{n-1}\times \mathbb{R}_{+}^{k}\mapsto \{0,1\}$ and we use $\Pi$ to denote the set of all such policies.

Suppose that the (random) lifetime of the user in the system is $T$, and let $N$ be the total number of items that arrive to the system before $T$.  Then our goal is to maximize:

\begin{equation}
\sup_{\pi\in\Pi}E^{\pi}\left[\sum_{n=1}^{N}U_{n}(Y_{n}-c)\right] \label{prob}
\end{equation}
where $E^{\pi}$ denotes the expected reward using policy $\pi$. 

For analytic tractability, we assume that $T$ is exponentially distributed, and let its rate parameter be $r>0$. Then, one can show that $N$ follows a geometric distribution with parameter $\gamma=\frac{\Gamma}{\Gamma+r}$, and the random finite horizon problem \eqref{prob} can be transformed to a discounted infinite horizon problem: 
\begin{equation}
E^{\pi}\left[\sum_{n=1}^{N}U_{n}(Y_{n}-c)\right]=\gamma E^{\pi}\left[\sum_{n=1}^{\infty} \gamma^{n-1} U_{n} (Y_{n}-c)\right],\label{Prob}
\end{equation}

where $\gamma=\frac{\Gamma}{\Gamma+r}$. The proof is the same as Lemma 1 in \cite{eeifp} and we omit the proof here.





\section{Main Results}
\label{sec:theory}
The problem described in section~\ref{sec:prob} is a partially observable Markov decision process, and can, in theory, be solved using stochastic dynamic programming, see \cite{ampomdp} and \cite{pomdp}.  However, the state space of this dynamic program on the belief state is in high dimension ($k$ dimensions are required to represent the posterior mean, and $O(k^2)$ dimensions are required for the posterior covariance matrix), which makes solving it computationally intractable. 

Instead, we provide in this section a computational upper bound of this problem (in Section~\ref{sec:bound}) and develop two implementable policies \acronym\ and \MUCB\ based on this upper bound in Section~\ref{sec:policy1} and Section~\ref{sec:policy2}. When \acronym\ and \MUCB, or any other implementable policy, gives us a result close to the upper bound, then we are reassured that this policy is nearly optimal.

In practice, \acronym\ and \MUCB\ tend to perform best when feature vectors are approximately aligned with a basis. This may tend to occur most frequently in high dimensional problems, where vectors tend to be orthogonal.

\subsection{Upper bound}\label{sec:bound}
In this section, we provide a computational upper bound on the value of the solution to \eqref{prob}. This upper bound is based on the idea of dividing \eqref{prob} into $k$ different ``single-feature'' subproblems, then performing an information relaxation (similar in spirit to \cite{irdsdp}) in which we give the policy assigned to each single-feature subproblem additional information, which allows us to compute their value efficiently.

Define  $Y_{i,n}=\theta_{i}+\epsilon_n^i$. Here $\epsilon_{n}^{i}\sim N(0,\frac{\lambda^{2}}{x_{i,n}^{2}})$ if $x_{i,n}>0$ and $\epsilon_{n}^{i}=0$ if $x_{i,n}=0$ for $i=1,\cdots,k$, independently distributed across i and n. We may think of $Y_{i,n}$ as the reward that we would have seen if $X_{n}$ were equal to $e_{i}$, where $e_{i}$ is a unit vector with the $i_{\mathrm{th}}$ element 1 and other elements 0. Later, we will use that
$Y_{n}=\sum_{i=1}^{k}x_{i,n}\theta_{i}+\epsilon_n =\sum_{i=1}^{k}x_{i,n}(\theta_{i}+\epsilon_{n}^{i})=\sum_{i=1}^{k}x_{i,n}Y_{i,n}.$

We will generalize the original problem \eqref{prob} by introducing notation that allows for separate forwarding decisions to be made for each feature. Define $U_{j,n}$ to be decision made for the \textrm{$j^{\mathrm{th}}$} feature of the \textrm{$n^{\mathrm{th}}$} item. The original problem \eqref{prob} can be recovered if we require that $U_{j,n}$ is identical across $j$ for each $n$.

For each feature $j$, we now introduce a new set of policies $\Pi_{j}$, which will govern the forwarding decisions $U_{j,n}$ for feature $j$, and under which these decisions can depend upon information not available in the original problem: they may depend on $\theta \cdot e_{i}$ for $\forall i\neq j$. Formally, the decision of whether or not to forward the \textrm{$j^{\mathrm{th}}$} feature of the \textrm{$n^{\mathrm{th}}$} item depends on the history $H_{n-1}^{j}=(U_{j,m},X_{j,m},U_{j,m}Y_{j,m}:m\leq n-1)$, our current $X_{j,n}$, and $\theta_{-j}=(\theta_{1},\cdots,\theta_{j-1},\theta_{j+1},\cdots,\theta_{k})$.  

Using these definitions, we may now state the computational upper bound. It bounds the value of the optimal policy for our original problem of interest \eqref{prob}, on the left-hand side, by the sum of a collection of values of single-feature problems, each of which have been given additional information.  Efficient computation of this right-hand side is discussed below, and summarized in Algorithm~\ref{algorithm1}.
\begin{theorem}
For $X_{n}$ that are bounded over all n, we have
\begin{align}
&\sup_{\pi\in\Pi}E^{\pi}\left[\sum_{n=1}^{N}U_{n}(Y_{n}-c)\right]  \nonumber \\
\leq &\sum_{j=1}^{k}\sup_{\pi^{''}\in\Pi_{j}}E^{\pi^{''}}\left[\sum_{n=1}^{N}U_{j,n}(x_{j,n}Y_{j,n}-\frac{x_{j,n}c}{\|X_{n}\|})\right], \nonumber 
\end{align}
where $\|X_{n}\|$ is the $L_1$ norm. When $\sum_{i=1}^{k}x_{i,n}=1$, then this theorem becomes:
\begin{align}
&\sup_{\pi\in\Pi}E^{\pi}\left[\sum_{n=1}^{N}U_{n}(Y_{n}-c)\right] \nonumber \\
\leq &\sum_{j=1}^{k}\sup_{\pi^{''}\in\Pi_{j}}E^{\pi^{''}}\left[\sum_{n=1}^{N}U_{j,n}(x_{j,n}Y_{j,n}-x_{j,n}c)\right]. \nonumber 
\end{align}
\label{t:bound}
\end{theorem}

\begin{proof}
Since $\|X_{n}\| = x_{1,n}+\cdots+x_{k,n}$, we know
\begin{align}
&\sup_{\pi\in\Pi}E^{\pi}\left[\sum_{n=1}^{N}U_{n}(Y_{n}-c)\right] \nonumber \\
=&\sup_{\pi\in\Pi}E^{\pi}\left[\sum_{n=1}^{N}U_{n}(x_{1,n}Y_{1,n}+\cdots+x_{k,n}Y_{k,n}-c)\right] \nonumber \\
=&\sup_{\pi\in\Pi}E^{\pi}\left[\sum_{n=1}^{N}\sum_{j=1}^{k}U_{n}(x_{j,n}Y_{j,n}-x_{j,n}\frac{c}{\|X_{n}\|})\right]. \label{upper1}
\end{align}
Now we introduce two new policy sets $\Pi_{0}^{'}$ and $\Pi^{'}$, which allow different features can make their own decisions $U_{j,n}$ for the \textrm{$n^{\mathrm{th}}$} item. Further, $\Pi^{'}_{0}$ has an additional restriction that $U_{1,n}=\cdots=U_{j,n}$. Based on the definition, we have
\begin{align}
(\ref{upper1})=&\sup_{\pi^{'}\in\Pi^{'}_{0}}E^{\pi^{'}}\left[\sum_{n=1}^{N}\sum_{j=1}^{k}U_{j,n}(x_{j,n}Y_{j,n}-x_{j,n}\frac{c}{\|X_{n}\|})\right] \nonumber \\
\leq & \sup_{\pi^{'}\in\Pi^{'}}E^{\pi^{'}}\left[\sum_{n=1}^{N}\sum_{j=1}^{k}U_{j,n}(x_{j,n}Y_{j,n}-x_{j,n}\frac{c}{\|X_n\|})\right]. \label{upper2}
\end{align}
Since the supremum of a summation is less or equal to the summation of a supremum, we have
\begin{align}
(\ref{upper2})\leq \sum_{j=1}^{k}\sup_{\pi^{'}\in\Pi^{'}}E^{\pi^{'}}\left[\sum_{n=1}^{N}U_{j,n}(x_{j,n}Y_{j,n}-x_{j,n}\frac{c}{\|X_{n}\|})\right]. \label{upper3}
\end{align}
Then based on the definition of our policy set $\Pi_{j}$, for $j=1,2,\cdots,k$, we know
\begin{align}
(\ref{upper3})\leq \sum_{j=1}^{k}\sup_{\pi^{''}\in\Pi_{j}}E^{\pi^{''}}\left[\sum_{n=1}^{N}U_{j,n}(x_{j,n}Y_{j,n}-x_{j,n}\frac{c}{\|X_n\|})\right], \nonumber 
\end{align}
which concludes the proof of the theorem.

\end{proof}

We emphasize that this computational upper bound holds true in general, even when the different components of $X_{n}$ are correlated. Numerical experiments in Section~\ref{sec:numerical} suggest that the optimality gap between this upper bound and the best heuristic policy is typically small. 

For simplicity, in this paper we focus on the special case where $\sum_{i=1}^{k}x_{i,n}=1$. We now discuss computation of the upper bound in Theorem~\ref{t:bound}. To compute this quantity, we must solve these $k$ subproblems:
\begin{align}
&\sup_{\pi\in\Pi_{j}}E^{\pi}\left[\sum_{n=1}^{N}U_{j,n}(x_{j,n}Y_{j,n}-x_{j,n}c)\right],  j=1,2,\cdots,k, 
\end{align}
where $Y_{j,n}|\theta_{j}\sim N(\theta_{j},\frac{\lambda^{2}}{x_{j,n}^{2}})$ and $\theta_{j}\sim N(\mu_{j,n}, \sigma_{j,n}^{2})$. Here $\theta_{j}\sim N(\mu_{j,n},\sigma_{j,n}^{2})$ represents our belief of $\theta_{j}$ after the first n items. 

Therefore for each subproblem, after the arrival of the \textrm{$n^{\mathrm{th}}$} item, we can update our parameters as the following:
 \[ \mu_{j,n} = \left\{ \begin{array}{ll}
         \frac{\lambda^{2}\beta_{j,n-1}\mu_{j,n-1}+Y_{j,n-1}x_{j,n-1}^{2}}{\lambda^{2}\beta_{j,n-1}+x_{j,n-1}^{2}} & \mbox{if $U_{j,n-1}=1$};\\
        \mu_{j,n-1} & \mbox{if $U_{j,n-1}=0$}.\end{array} \right.  \]
The precision of our beliefs (which is the inverse of the prior/posterior variance with initial value $\beta_{j,0}=\frac{1}{\sigma_{j,0}^{2}}$) is updated as follows:
 \[ \beta_{j,n} = \left\{ \begin{array}{ll}
         \beta_{j,n-1}+\frac{x_{j,n-1}^{2}}{\lambda^{2}} & \mbox{if $U_{j,n-1}=1$};\\
        \beta_{j,n-1} & \mbox{if $U_{j,n-1}=0$}.\end{array} \right. \]

The \textrm{$j^{\mathrm{th}}$} single-feature subproblem can be solved using dynamic programming with a three-dimensional state space $(\mu_{j,n},\sigma_{j,n},x_{j,n})$, where $\mu_{j,n}$ and $\sigma_{j,n}$ are the mean and variance of our current belief about $\theta_{j}$ and $x_{j,n}$ is the current item's \textrm{$j^{\mathrm{th}}$} feature. Initially, $\mu_{j,0}$ and $\sigma_{j,0}$ are given by the conditional distribution of $\theta_{j}$ given $\theta_{-j}$ and the prior distribution $\theta\sim N(\mu,\Sigma)$. Upon each item's arrival, we move to another state based on the updating formula described above. Define $Q_{j}(\mu,\sigma,x,0)$ and $Q_{j}(\mu,\sigma,x,1)$ be the total reward to go if you decided to discard the item and forward the item respectively, 
\begin{align}
Q_{j}(\mu,\sigma,x,U)=&\sup_{\pi^{''}\in \Pi_{j}}E^{\pi^{''}
}[\sum_{n=1}^{\infty}\gamma^{n-1}U_{j,x}(x_{j,n}Y_{j,n}-x_{j,n}c) \nonumber \\
&|\theta_{j}\sim N(\mu,\sigma^{2}),x_{j,1}=x, U_{j,1}=U]. \nonumber 
\end{align}

Then the Bellman equation for this problem is:
\begin{align}
    V_{j}(\mu,\sigma,x) = \max_{U=0,1}Q_{j}(\mu,\sigma,x,U). \label{bellman}
\end{align}

This calculation is summarized as Algorithm~\ref{algorithm1}.

\begin{algorithm}
\caption{Calculation of the \textrm{$j^{\mathrm{th}}$} subproblem}\label{algorithm1}
\begin{algorithmic}
\STATE Solve the dynamic program using backward induction  (discretizing and truncating), with state space $(\mu_{j,n},\sigma_{j,n},x_{j,n})\in \mathbb{R}\times\mathbb{R}^{+}\times[0,1]$, 
infinite horizon and value function $V_{j}(\mu,\sigma,x)$. 

\FOR {$i=1 ; i<M; i++$}  
{
\STATE Generate $\theta\sim N(\mu,\Sigma)$; \STATE Calculate the conditional distribution of $\theta_{j}\sim N(\mu_{j,0},\sigma_{j,0})$, given $\theta\sim N(\mu,\Sigma)$ and $\theta_{-j}$. \STATE Generate $x_{j,0}$ from the distribution of $X_{n}$. 
\STATE Find the optimal value of state $(\mu_{j,0},\sigma_{j,0},x_{j,0})$ and denote it as $V_{i}$.
}\ENDFOR

\STATE Calculate $\bar{V}=\frac{1}{M}\sum_{i=1}^{M}V_{i}$ and use \eqref{Prob} to get the optimal value for the $j^{\mathrm{th}}$ subproblem, where M is the number of simulation.
\end{algorithmic}
\end{algorithm}

We may improve our upper bound by taking its minimum with a hindsight upper bound, derived in the following way. 
We first consider a larger class of policies that may additionally base their decisions on full knowledge of $\theta$.  An optimal policy among this larger class of policies forwards the $n^{\mathrm{th}}$ item to the user only if $\theta\cdot X_{n}>c$, and the expected total reward of this optimal policy is 
\begin{equation}
E\left[\sum_{n=1}^N (\theta\cdot X_{n}-c)^{+}\right]
= \frac\gamma{1-\gamma} E\left[(\theta\cdot X_1 - c)^+\right].
\label{eq:hindsight}
\end{equation}
Since \eqref{eq:hindsight} is the supremum of the same objective as \eqref{Prob}, but over a larger set of policies, it forms an upper bound.
This style of analysis was also applied in \cite{ssesp}.
In Section~\ref{sec:numerical}, we use the minimum of the computational upper bound in Theorem~\ref{t:bound} and the hindsight upper bound \eqref{eq:hindsight} as our theoretical upper bound. 

\subsection{The \acronym\ policy}
\label{sec:policy1}
The analysis in Section~\ref{sec:bound} provides a way to bound the performance of any policy, and is derived by decomposing the original multi-feature problem into many single-feature subproblems.  
In this section, we build on this same idea to develop an implementable policy, called \acronym, and in Section~\ref{sec:policy2} we build on this idea further to create a second implementable policy, called \MUCB.

In \acronym, as each item arrives, we consider the decomposition from Section~\ref{sec:bound} taking the incoming feature vector $X_n$ and choosing a basis for which 
$X_n$ is a unit vector in the basis. 
This basis may change with each $n$.

We then consider the decomposed problem studied in Section~\ref{sec:bound}, in which we may make separate forwarding decisions for each direction in the basis, and compute the value of exploration corresponding to $X_n$ in this decomposed problem.  

To compute this value of exploration, we first compute the distribution of the magnitude $x$ of the projection of future feature vector $X$ along direction $X_n$,  $x=\frac{X_{n}\cdot X}{X_n\cdot X_n}$, by using the distribution of future feature vectors $f(X)$.  Denote this distribution by $G(x|X_n)$. We then solve the corresponding single-feature subproblem using \eqref{bellman} as described in Section~\ref{sec:bound}.

From this solution, we derive Q factors, 
$Q(\mu_{1,0},\sigma_{1,0},x_0,0)$ and 
$Q(\mu_{1,0},\sigma_{1,0},x_0,1)$
corresponding to the value of discarding and forwarding the current item in the single feature subproblem, given that the current feature vector has magnitude $x_0=1$ and given that our current prior mean and variance for the subproblem are
\begin{align}
    \mu_{1,0} = X_{n}\cdot \mu_{n},
    \sigma_{j,0}^{2} = X_{n}\Sigma_{n}X_{n}^{T}. \nonumber 
\end{align}

We then define the ``exploration benefit'' $E(\mu_{1,0},\sigma_{1,0})$ from forwarding the current item as the overall benefit of forwarding, minus the myopic benefit of forwarding $\mu_{1,0} - c$ and the benefit of discarding:
\begin{align*}
       E(\mu_{1,0},\sigma_{1,0}) = & Q(\mu_{1,0},\sigma_{1,0},1,1)- \\
    &Q(\mu_{1,0},\sigma_{1,0},1,0)-\mu_{1,0} + c.
\end{align*}

In \acronym, we add a scalar tuning parameter $\alpha$, mirroring the tuning parameter used in UCB, to scale up or down the exploration benefit.  The default value for $\alpha$ is $\alpha=1$.  Then, returning to the original multi-dimensional problem, we  consider the net benefit of forwarding to be the myopic benefit $X_n \cdot \mu_n - c$ plus the exploration benefit $\alpha E(\mu_{1,0},\sigma_{1,0})$, and forward when this is strictly positive.
This is summarized in Algorithm~\ref{algorithm2}.

\begin{algorithm}
\caption{The \acronym\ algorithm}\label{algorithm2}
\begin{algorithmic}
\FOR{$n=1,2,\cdots$}
{
    \STATE{
     Denote $\mu_{1,0} = X_{n}\cdot \mu_{n}$ and  $\sigma_{1,0}^{2} = X_{n}\Sigma_{n}X_{n}^{T}$; \\
     Calculate $Q(\mu_{1,0},\sigma_{1,0},1,U)$ for $U=0,1$ given that $x\sim G(x|X_n)$; \\
     Denote $E(\mu_{1,0},\sigma_{1,0})=Q(\mu_{1,0},\sigma_{1,0},1,1)-Q(\mu_{1,0},\sigma_{1,0},1,0)-\mu_{1,0}+c$; \\
     \IF{$\mu_{1,0}+\alpha\cdot E(\mu_{1,0},\sigma_{1,0})>c$}
     \STATE{Forward the item}
     \ELSE 
     \STATE{Discard the item}
     \ENDIF
     }
     
}\ENDFOR

\end{algorithmic}
\end{algorithm}


\subsection{The \MUCB\ algorithm}
\label{sec:policy2}

In this section, we develop a second heuristic, \NMUCB\ (\MUCB), which builds on the ideas underlying \acronym.

In \acronym, we considered a single-feature subproblem in which the magnitude $x$ of the projection of future feature vectors is given by $G(x|X_n)$ and in which the prior mean and prior variance were given by $X_n \cdot \mu_n$ and $X_n \Sigma_n X_n^T$ respectively.  We then quantified the value of exploration by solving the single-feature subproblem using stochastic dynamic programming. In this single-feature subproblem, we observe that when future feature vectors are more closely aligned with $X_n$, so that samples from $G(x|X_n)$ are large, we are more willing to explore.

In our second heuristic \MUCB, we take a similar approach, but quantify the value of exploration using an approach adopted from the literature on upper confidence bound policies, which quantifies the value of exploration in terms of some scalar multiple $\alpha$ of the standard deviation of the value of an action, obtained from calculating an upper confidence bound and subtracting the center of the confidence region.
In \MUCB, we quantify the value of information similarly, but add an additional scaling factor to include the fact that those $X_n$ whose $G(x|X_n)$ have larger moments should induce more exploration.


To accomplish this, we let $M(X_n)$ be the mean of the distribution $G(x|X_n)$.  This ``mean of the projection'' is
\begin{align}
    M(X_n) = \int_{X}\frac{X_n\cdot X}{X_{n} \cdot X_n}f(X)dX. \nonumber 
\end{align}

We summarize the \MUCB\ algorithm in Algorithm~\ref{algorithm3}.

\begin{algorithm}
\caption{The \MUCB\ algorithm}\label{algorithm3}
\begin{algorithmic}
\FOR{$n=1,2,\cdots$}
{
\IF{$X_n \cdot \mu_n+\alpha \cdot M(X_n)\cdot \sqrt{X_n \Sigma_n X_n}>c$}

\STATE {Forward the item}
\ELSE 
\STATE {Discard the item}
\ENDIF

}\ENDFOR

\end{algorithmic}
\end{algorithm}


\section{Numerical Experiments}
\label{sec:numerical}

In this section, we compare \acronym\ and \MUCB\ with three different benchmark algorithms and the computational upper bound from Section~\ref{sec:bound} using both real and simulated data.
The benchmark algorithms are:
\begin{itemize}
\item Pure Exploitation: Forward the item if $X_{n}\cdot \mu_{n}\geq c$.
\item Upper Confidence Bound (UCB): Forward the item if $X_{n}\cdot \mu_{n}+\alpha\sqrt{X_{n}\Sigma_{n}X_{n}^{T}}\geq c$. 
\item Linear Thompson Sampling (LTS): For item $X_{n}$, generate $\theta\sim N(\mu_{n},\Sigma_{n})$. Forward the item if $\theta\cdot X_{n}>c$.
\end{itemize}

For \acronym, \MUCB\ and UCB, there is a tuning parameter $\alpha$. In our simulation experiments we run these policies with 10 different values of $\alpha$ ranging from $0.1$ to $10$ on a log scale, and display the one with the best performance (which requires simulating performance for different values of $\alpha$ in a Monte Carlo simulation as a pre-processing step) in each instance.


We evaluate our upper bound and proposed policy on real and simulated data, and find our upper bound is tight enough to be useful (the best policy evaluated is often within 60\% of the upper bound and never below 30\% of the upper bound).

\begin{figure*}[t!]
    \centering
    \begin{subfigure}[t]{0.5\textwidth}
        \centering
        \includegraphics[width=1\textwidth]{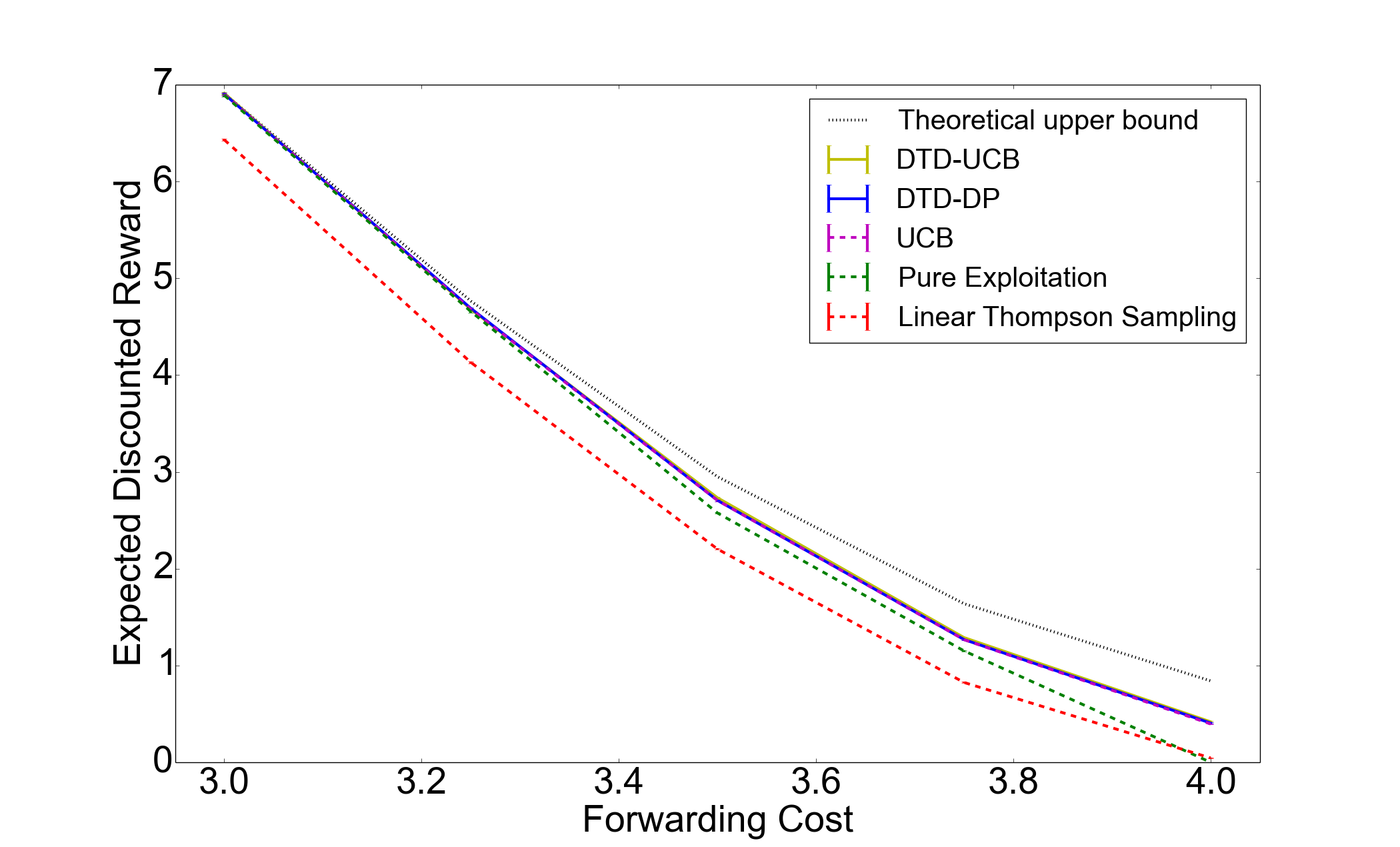}
        \caption{Comparison of Different Policies Using Yelp Academic Data}   
        \label{fig:TR_yelp} 
        \end{subfigure}%
    ~ 
    \begin{subfigure}[t]{0.5\textwidth}
        \centering
        \includegraphics[width=1\textwidth]{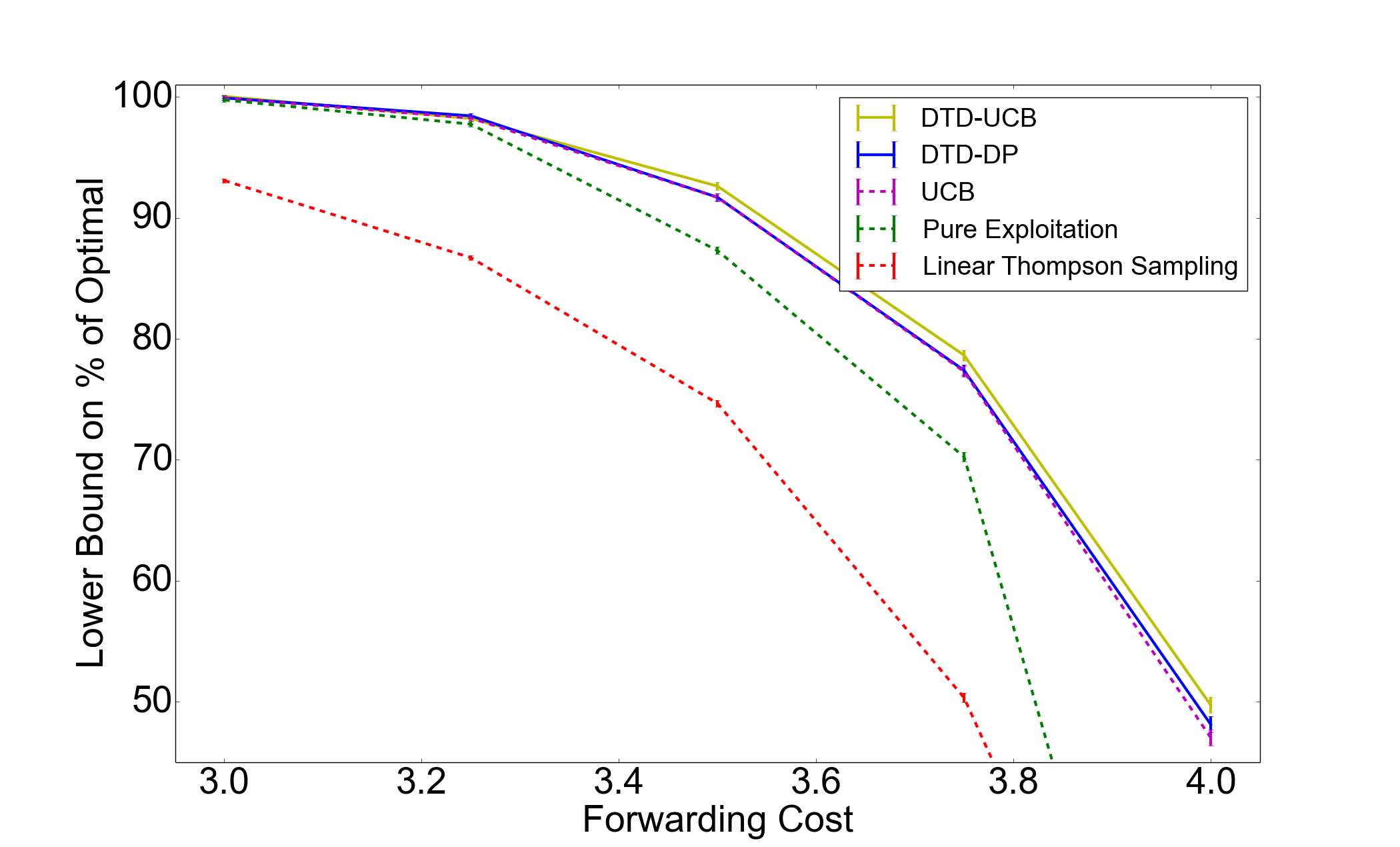}
        \caption{Optimality Gap of Different Policies Using Yelp Academic Data}
        \label{fig:OG_yelp}
    \end{subfigure}
    \caption{The performance of \acronym, \MUCB\ and three benchmark algorithms relative to the computational upper bound.
This plot compares performance on the Yelp academic dataset (Section~\ref{sec:realdata1}), and shows that \MUCB\ outperforms all other heuristic policies. \acronym\ performs comparably (and nearly identical to) UCB, and outperforms pure exploitation and LTS. \MUCB\ performs close to the computational upper bound, showing their performance is close to optimal. \label{fig:result1}}
\end{figure*}

\begin{figure*}[t!]
    \centering
    \begin{subfigure}[t]{0.5\textwidth}
        \centering
        \includegraphics[width=1\textwidth]{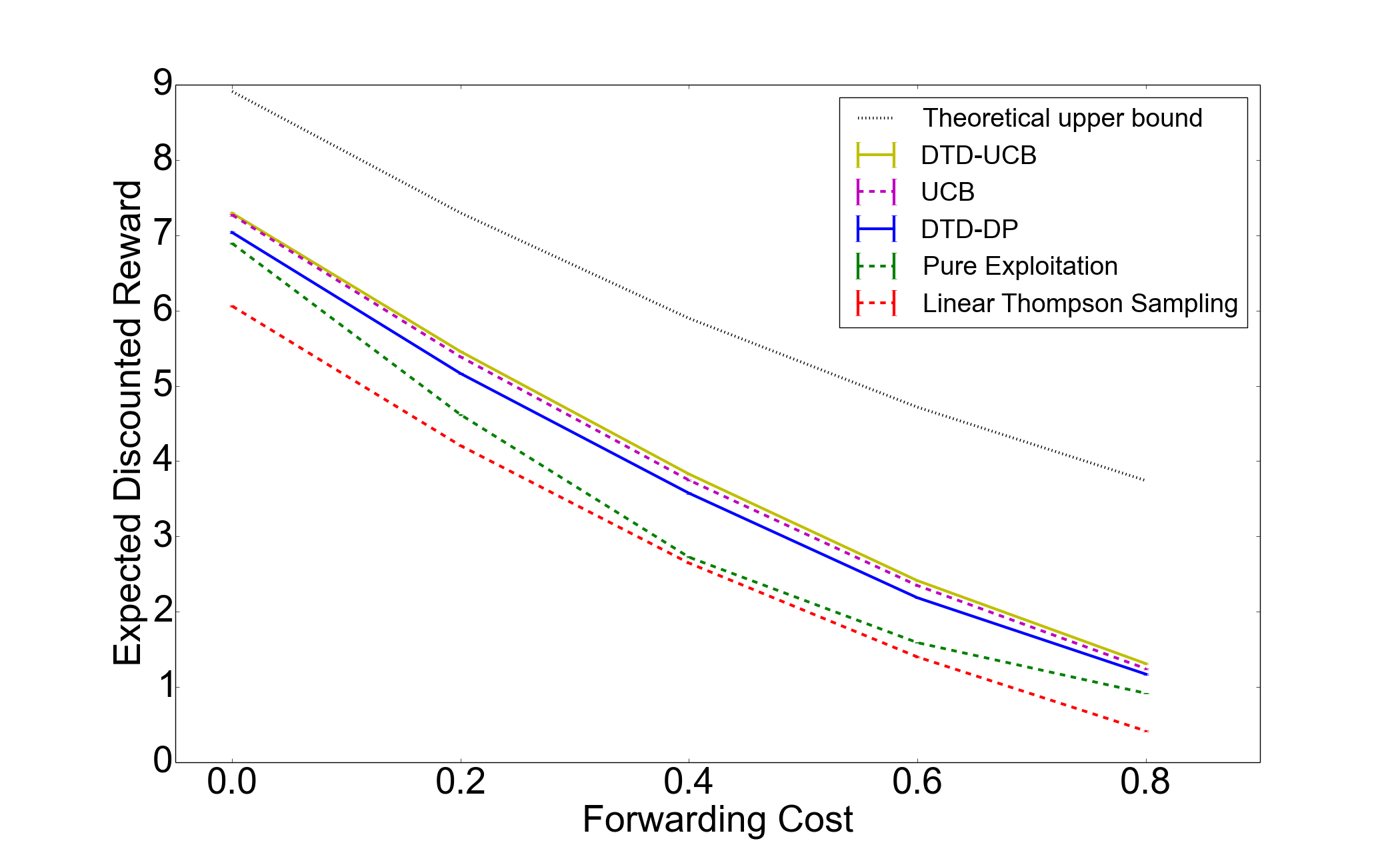}
        \caption{Comparison of Different Policies Using arXiv.org dataset}   
        \label{fig:TR_arxiv} 
        \end{subfigure}%
    ~ 
    \begin{subfigure}[t]{0.5\textwidth}
        \centering
        \includegraphics[width=1\textwidth]{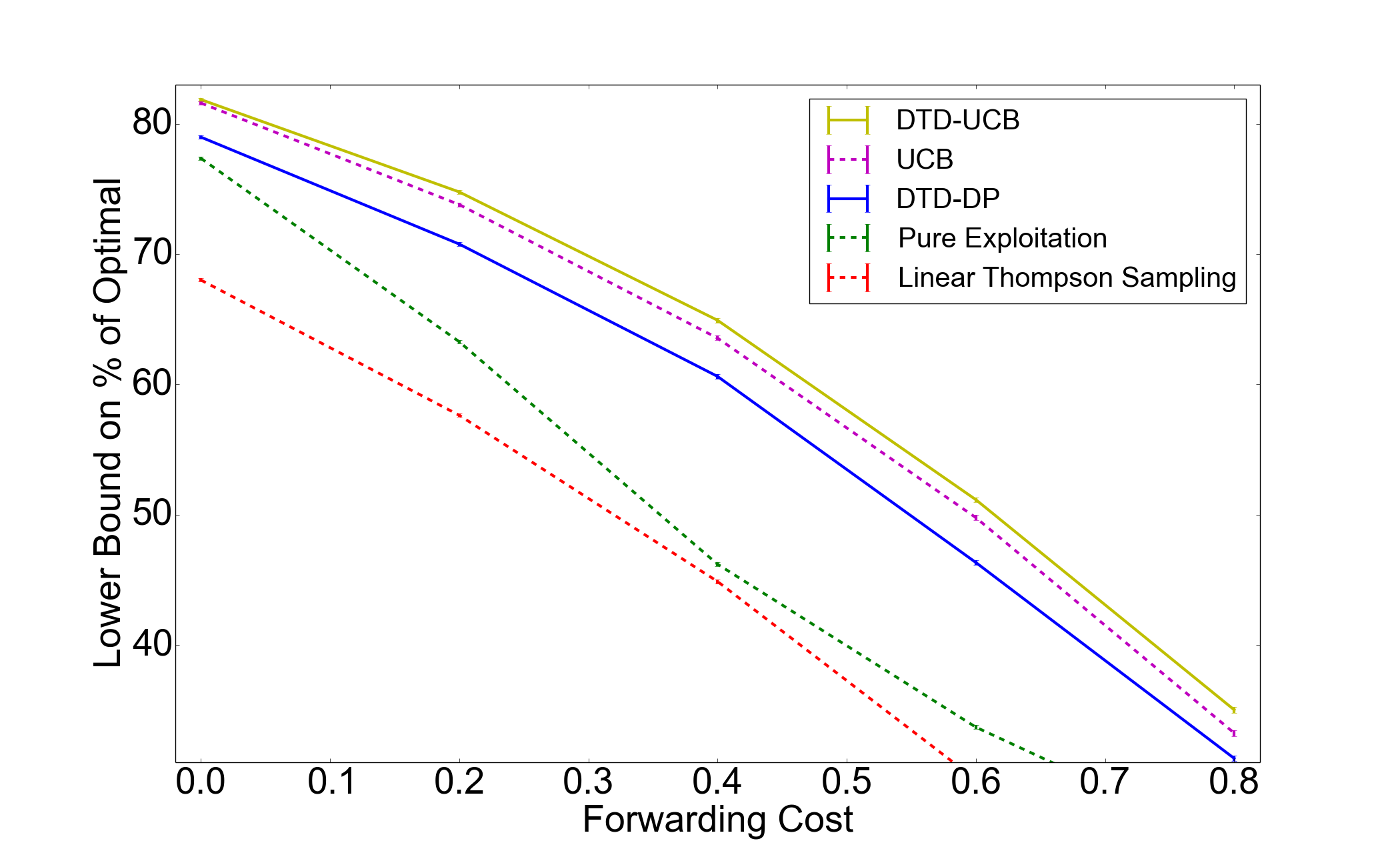}
        \caption{Optimality Gap of Different Policies Using arXiv.org dataset}
        \label{fig:OG_arxiv}
    \end{subfigure}
    \caption{The performance of \acronym, \MUCB\ and three benchmark algorithms relative to the computational instance-specific upper bound.
This plot compares performance on the 2014 arXiv.org Condensed Matter dataset (Section~\ref{sec:realdata2}), and shows that \MUCB\ outperforms all other heuristic policies. \label{fig:result2}}
\end{figure*}

\begin{figure*}[t!]
    \centering
    \begin{subfigure}[t]{0.5\textwidth}
        \centering
        \includegraphics[width=1\textwidth]{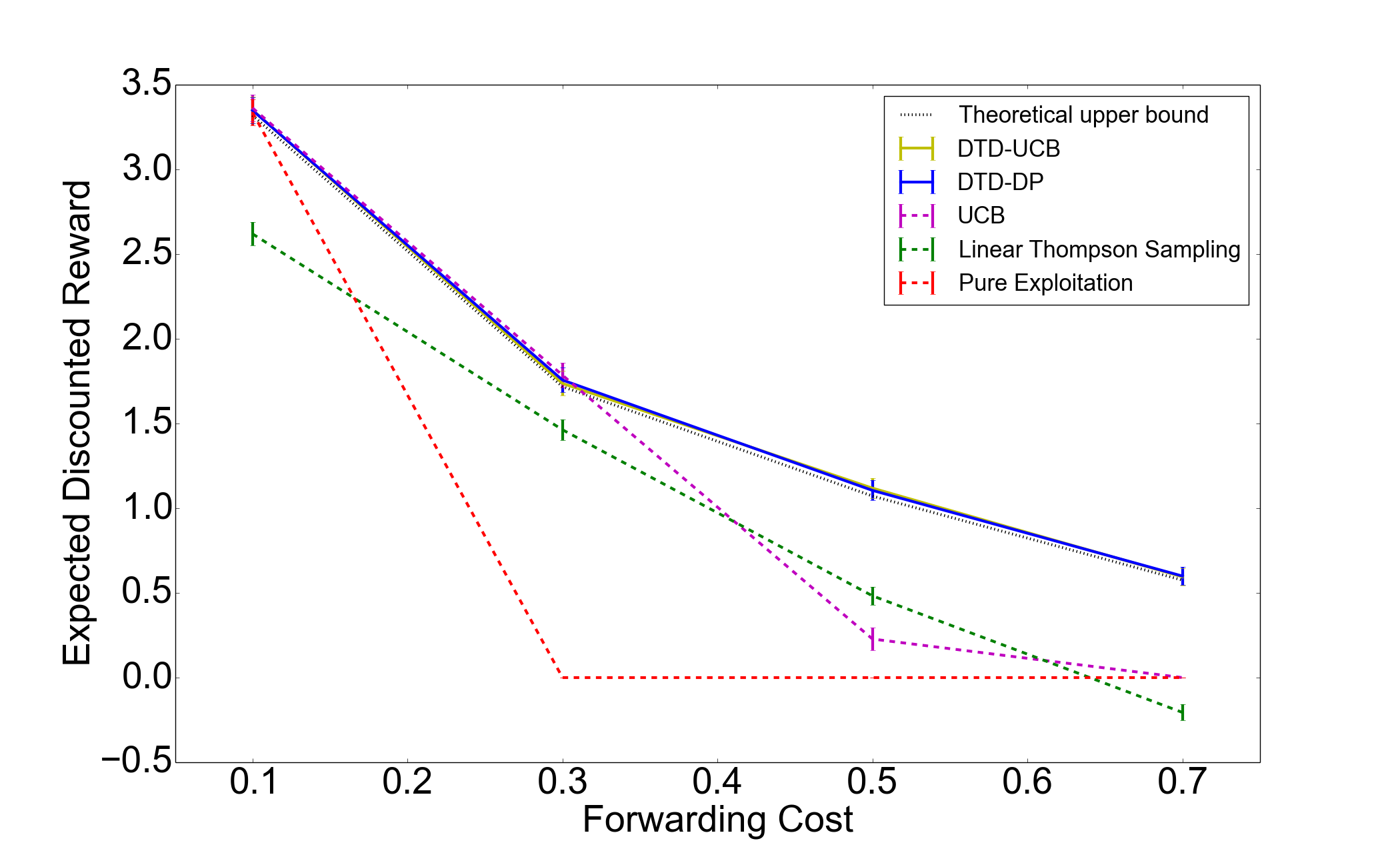}
        \caption{Comparison of Different Policies Using Simulated Data}   
        \label{fig:TR_ideal} 
        \end{subfigure}%
    ~ 
    \begin{subfigure}[t]{0.5\textwidth}
        \centering
        \includegraphics[width=1\textwidth]{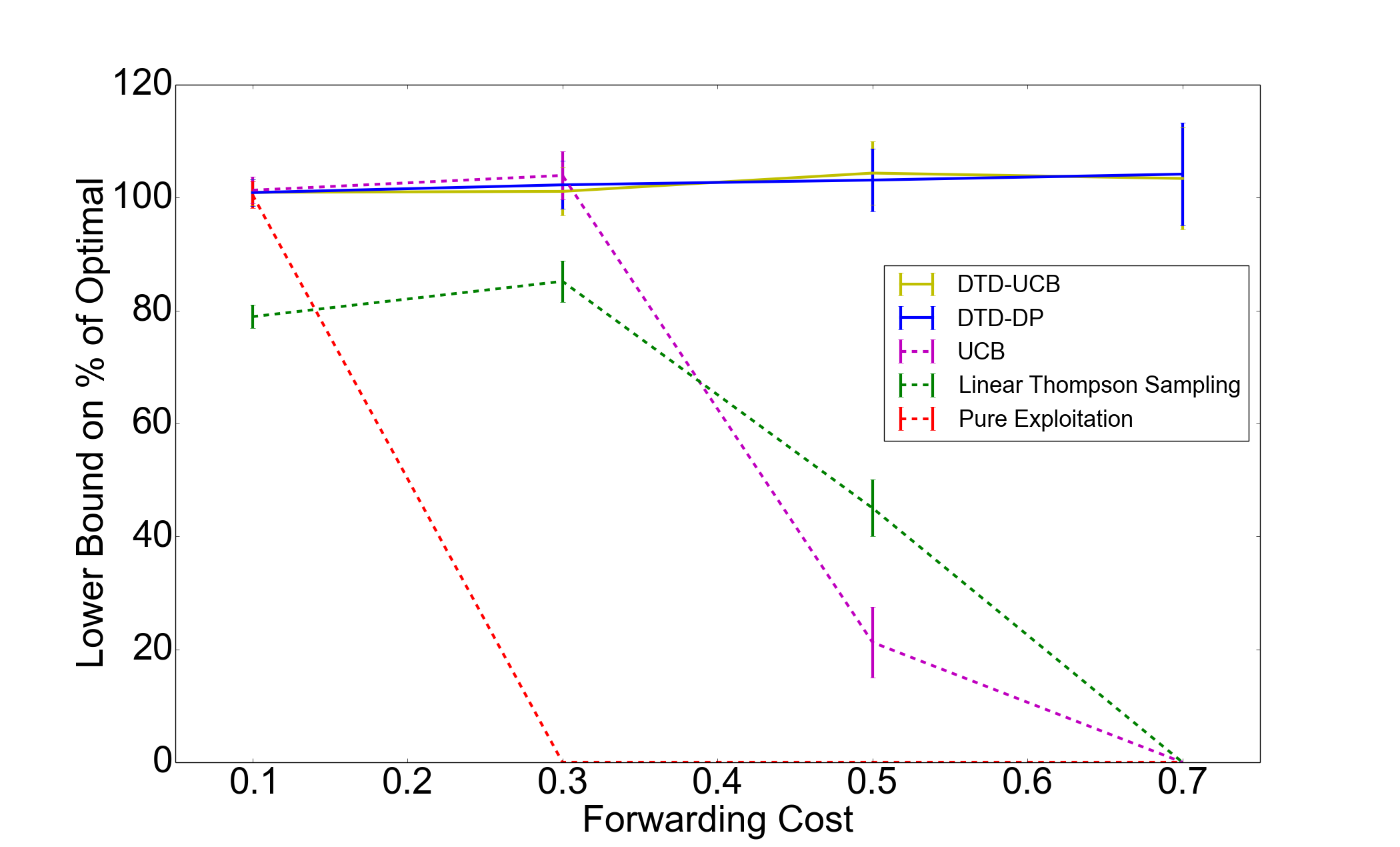}
        \caption{Optimality Gap of Different Policies Using Simulated Data}
        \label{fig:OG_ideal}
    \end{subfigure}
    \caption{
The performance of \acronym, \MUCB\ and three benchmark algorithms relative to the computational instance-specific upper bound using simulated data.
This plot compares performance on simulated data (Section~\ref{sec:simudata}), and shows that \acronym\ and \MUCB\ outperform all the other algorithms and coincides with the theoretical upper bound, showing it is indistinguishable from optimal in this case.}
\label{fig:result3}
\end{figure*}

\subsection{Yelp academic data}
\label{sec:realdata1}
In this section, we compare \acronym\ and \MUCB\ against benchmarks using the Yelp academic dataset \cite{yelp}.

Our items are businesses, and are described as belonging to one or more of the following six categories: Restaurants, Shopping, Food, Beauty and Spas, Health and Medical and Nightlife.
The $j^{\mathrm{th}}$ business object is then described by a 6-dimensional feature vector $X_{j}=(x_{1,j},x_{2,j},\cdots,x_{6,j})$ with the $i^{\mathrm{th}}$ element $x_{i,j}=1$ if the business belongs to category i, and $x_{i,j}=0$ otherwise. 
Then we normalize $X_{j}$ such that its L1 norm is 1.

We calculate the prior distribution over new customers' preferences using historical users' reviews. For each historical user, we use linear regression to regress his reviews' ratings on the feature vectors of the business objects that he reviewed. We use the estimated linear regression coefficients as his/her true user preference vector. Then we calculate the empirical distribution for all historical users, and set the prior on new users' preference vectors to be multivariate normal with mean vector and covariance matrix equal to the sample mean and sample covariance of the historical users.

In Figure~\ref{fig:TR_yelp}, evaluation is done by taking a collection of real historical users, and for each estimating his true preference vector $\theta$ using linear regression on historical data. Evaluation is then performed for each algorithm and user by simulating feedback from the user's held out $\theta$ on items forwarded by the algorithm, and an algorithm's average performance is calculated by averaging across users.
We must simulate user feedback given $\theta$ because we do not have historical relevance feedback from all users for all items, and algorithms may present items that have not been rated. We plot the $95\%$ confidence interval of cumulative reward over 100 items forwarded to the user with discount factor $\lambda = 0.9$.

In Figure~\ref{fig:OG_yelp}, we calculate the optimality gap between each heuristic algorithm and our computational upper bound. A smaller gap suggests the corresponding policy performs better in this problem instance.


The plot in Figure~\ref{fig:result1} summarizes the results. In this problem instance, \MUCB\ outperforms \acronym, UCB, pure exploitation and LTS, with \acronym\ and UCB performing almost identically.  Moreover, the optimality gap is relatively small, which shows that \MUCB\ performs close to optimal.  



\subsection{arXiv.org Condensed Matter Dataset}
\label{sec:realdata2}
In this section, we compare \acronym\ and \MUCB\ with benchmarks using readership data from articles submitted in 2014 to the arXiv condensed matter category. We represent each paper submitted in 2014 by a 10 dimensional vector using Latent Dirichlet allocation (LDA) \cite{lda}. For each user, the rating for a paper is 1 if he/she clicks and otherwise the rating is 0. We then calculate the user's preference vector by linear regression. Similar to Section~\ref{sec:realdata1}, we use the sample mean and sample variance of users' preference vectors as our prior distribution parameters.

In our simulation, we use true users' preference vectors calculated using linear regression, as we did in Section~\ref{sec:realdata1}. For each user, we randomly pick 100 papers and make the forwarding decisions using different policies. We evaluate the cumulative reward for these 100 papers with discount factor $\lambda = 0.9$.

The result is summarized in Figure~\ref{fig:result2}. The best of our heuristic policies in this example, \MUCB, outperforms all other heuristic policies. In this specific example, \acronym\ does not perform as well as UCB but it outperforms pure exploitation and LTS. 

\subsection{Simulated Data}
\label{sec:simudata}
In this section, we compare the performance of \acronym\ and \MUCB\ with three benchmark algorithms, as well as our computational upper bound on simulated data.
This simulated data is chosen to give insight into situations where UCB can underperform, and where the structure of a policy like \acronym\ and \MUCB\ are needed to provide near-optimal performance.
We emphasize that it is chosen to provide insight, and not to show performance on a typical real problem instance --- we refer this comparison to Section~\ref{sec:realdata1} and Section~\ref{sec:realdata2}.

Each item is described by a 100-dimensional feature vector $X_{n}$ with the following distribution: $P(X_{n}=e_{1})=\frac{100}{199}$, $P(X_{n}=e_{i})=\frac{1}{199}$ for $i=2,\cdots,100$. Here, $e_x$ is the unit vector in the $x$th dimension.  The initial belief on the user's preference for each feature is $N(0.3,1.0)$ with independence across features. We set $\gamma=0.9$ and $\lambda=0.1$.  In estimating the infinite-horizon discounted sum \eqref{Prob}, we truncate after $n=100$.

The results, summarized in Figure~\ref{fig:result3}, show that \acronym\ and \MUCB\ outperform UCB, pure exploitation and LTS. In most cases, UCB performs very well with a properly chosen $\alpha$. Moreover, \acronym\ and \MUCB\ outperform UCB for several values of the forwarding cost, and nearly coincides with the theoretical upper bound for all values of the forwarding cost, which shows that it is indistinguishable from optimal in this problem instance.

LTS does not perform well in this example because it performs poorly at the initial stages and the (discounted) reward in the later stages cannot make up for the loss at the early stages. As \cite{lovis} and \cite{lovps} pointed out, LTS generally underperforms tuned UCB.

UCB underperforms \acronym\ and \MUCB\ in this example because it cannot account for the frequency with which a feature appears, and thus cannot adjust its level of exploration (encoded as the choice of $\alpha$) to explore more those features that tend to reoccur frequently, and explore less those features that are unlikely to appear again. In contrast, both \acronym\ and \MUCB\ can adjust its level of exploration, and will explore more those features that will reoccur.  

\section{Conclusion}
We studied the Bayesian linear information filtering problem, providing an instance-specific computational upper bound and a pair of new {\it Decompose-Then-Decide} heuristic policies, \acronym\ and \MUCB. Numerical experiments show that the best of these two policies is typically close to the computational upper bound and outperforms several benchmarks on real and simulated data.

\section*{Acknowledgment}
The authors were partially supported by NSF CAREER CMMI-1254298, NSF CMMI-1536895, NSF IIS-1247696, NSF DMR-1120296, AFOSR FA9550-12-1-0200, AFOSR FA9550-15-1-0038,  and AFOSR FA9550-16-1-0046.

\end{document}